\documentclass[runningheads]{llncs}

\usepackage[square,sort,comma,numbers]{natbib}
\usepackage{amsfonts}
\usepackage{algorithmic}
\usepackage[ruled, vlined]{algorithm2e}
\usepackage{graphicx}
\usepackage{color}
\usepackage{amssymb} %, amsthm
\usepackage{psfrag}
\usepackage[normalem]{ulem}
\usepackage{ifthen, latexsym}
\usepackage{amsmath}
\usepackage{cases}
\usepackage{array}
\usepackage{enumerate}
\usepackage{subcaption}
\usepackage{relsize}
\usepackage{verbatim}
\usepackage{float}
\usepackage{url}

\usepackage{natbib}
\usepackage[belowskip=2pt,aboveskip=1.5pt]{caption}
 
\makeatletter
\renewcommand\bibsection%
{
  \section*{\refname
    \@mkboth{\MakeUppercase{\refname}}{\MakeUppercase{\refname}}}
}
\makeatother

\title{Multi-robot Symmetric Rendezvous Search\\ on the Line}

\author{Deniz Ozsoyeller\inst{1} \and Pratap Tokekar\inst{2}}

\institute{Department of Software Engineering, Yaşar University, Izmir, Turkey. \email{deniz.ozsoyeller@yasar.edu.tr}  \and Department of Computer Science, University of Maryland, College Park, MD, USA. \email{tokekar@umd.edu}}

\def\boun{boun\textsubscript{t}}
\def\notM{\overline{M}}

\def\notS{\overline{S}}

\def\MSR{\mathcal{MSR}}

\def\dir{dir_{1}}
\def\dird{dir_{d}}
\def\notdir{\overline{dir_{1}}}

\newcommand{\Pee}[0]{\ensuremath{{\mathbb P}}}
\newcommand{\Ee}[0]{\ensuremath{{\mathbb E}}}

\begin{document}

\maketitle
\begin{abstract}
We study the Symmetric Rendezvous Search Problem for a multi-robot system. There are $n>2$ robots arbitrarily located on a line. Their goal is to meet somewhere on the line as quickly as possible. The robots do not know the initial location of any of the other robots or their own positions on the line. 
% Furthermore, they are unaware of the initial positions of their own and each other. 
The symmetric version of the problem requires the robots to execute the same search strategy to achieve rendezvous. 
Therefore, we solve the problem in an online fashion with a randomized strategy. In this paper, we present a symmetric rendezvous algorithm which achieves a constant competitive ratio for the total distance traveled by the robots. We validate our theoretical results through simulations.
\end{abstract}

\section{Introduction}
\label{sec:intro}

There are various examples of the rendezvous search problem in real life: two friends who are separated in a shopping mall and want to meet again; a group of parachutists who wants to meet after landing in a large field; a rescue helicopter looking for a lost hiker in the desert \cite{thomas1997searching}. The common challenges in each example are that the searchers are unaware of the location of the others and no common meeting point has been decided a priori. As such, the searchers need to execute an online rendezvous search to meet at a common location as quickly as possible.

In this paper, we study the robotic version of rendezvous search problem: a team of robots whose locations are unknown to each other should meet somewhere in the environment in the least  time possible. The robots have limited sensing capabilities and are operating in a large environment. The robots can only communicate when they are in close proximity of each other; i.e., the communication is only possible when the robots meet each other. There are many reasons why the robots may need to rendezvous. For example, in a scenario where the robots need share information with the others urgently, a rendezvous algorithm would lead to better coordinated planning~\cite{meghjani2012multi}.

There are two primary versions of the rendezvous search problem, depending on whether or not the robots can meet in advance of the search to agree on
the strategies that they will execute. In \emph{asymmetric} rendezvous search, the robots can meet in advance and choose distinct roles for each robot. For example, one robot can wait while the other carries out an exhaustive search. This is different from \emph{symmetric} rendezvous search, where the robots execute the same strategy, since they do not have a chance to agree on their roles. In this version, it is not necessary to implement a different strategy on each robot; thus, this makes it more appealing for programming large multi-robot systems.

In this paper, we study the symmetric version of the rendezvous search problem with $n > 2$ robots that are arbitrarily located on a line.
For example, robots may be deployed in a linear environment such as a road, a corridor, a river, or a tree row.
We consider a scenario where the robots are unaware of their own position along the line or the positions of any of the other agents. In fact, we consider the even more challenging scenario where even the initial distance between any pair of robots is unknown. We also do not assume that the robots know the directions leading to the other robots (i.e., a robot does not know whether the other robots are to its left, or right, or both). Each robot can only keep track of their own positions relative to their own starting positions. 
% We consider that the initial distance between any pair of robots is unknown. Moreover, the robots do not know their own positions and that of the other robots. 
% \PRT{I changed this a bit. Please check if its correct.}
% They are also unaware of the directions leading to each other.
In the absence of any prior knowledge and global information, we propose an online strategy to be followed by all the robots. The strategy involves making random choices in the directions to move to break the symmetry of the search. 

We analyze its performance using the notion of competitive ratio. The competitive ratio of an online algorithm is the worst-case ratio of the cost of the solution (i.e., distance traveled by the robots) found by the online algorithm to the cost of an optimal offline solution. For omniscient robots, the optimal offline algorithm would be to move toward the midpoint of the line segment with endpoints at the positions of the leftmost and rightmost robot. If an online algorithm has a constant competitive ratio, then it means it performs competitively with respect to omniscient robots even in the absence of prior knowledge.

The rendezvous search problem has been extensively studied in literature for linear search environments with the focus mostly on two-players rendezvous~\cite{anderson1995rendezvous, baston1999note, han2008improved, stachowiak2009asynchronous, baston1998rendezvous, ozsoyeller2013symmetric,alpern1995rendezvous,
gal1999rendezvous, alpern1999asymmetric, alpern2000pure}.
% \PRT{Would be good to cite a paper or two for 2-player rendezvous here} 
The studies on multi-player rendezvous \cite{lim1996minimax, lim1997rendezvous, gal1999rendezvous} assume that the initial distances between the robots are known and the robots are placed equidistant apart. For the symmetric version, it is also assumed that when two players meet, they may exchange any information known to them at the time. In contrast, we do not make any of these assumptions in our study. Our main contribution in this paper is to study the symmetric version of the multi-robot rendezvous problem with unknown initial distance between any pair of robots. We present a randomized symmetric rendezvous algorithm which yields a constant competitive ratio.

% The paper is organized as follows. We discuss the relevant works in Section~\ref{sec:related}. We present the formulation of the multi-robot rendezvous search problem in Section~\ref{sec:formulation} and  our algorithm in Section~\ref{sec:msralgorithm}. In Section~\ref{sec:msranalysis}, we analyze its performance. Finally, we provide the concluding remarks in Section~\ref{sec:conclusion}. 

\section{Related Work}
\label{sec:related}

The rendezvous search problem generalizes the \emph{linear search problem} in which a single searcher tries to find a stationary object located at an unknown initial distance. This problem is also known as the \emph{lost cow problem} and was originally solved in \cite{beck1970yet}. That solution was rediscovered in \cite{baezayates1993searching}. In the formulation of the lost cow problem, a near-sighted cow searches for the only gate in a long, straight fence. This gate is located at an unknown initial distance $d$, possibly to the left or right of the cow. The LostCow solution consists of the cow alternately searching to its left and then to its right starting from its initial location and doubling the distance it travels in each round until it finds the gate. This algorithm has a competitive ratio of $9$ which is the best possible performance for a deterministic online algorithm. Beck and Newman \cite{beck1970yet} showed that introducing some randomness reduces this competitive ratio to $4.591$. The same result was also provided by Kao et al. \cite{kao1996searching}. Chrobak et al.~\cite{chrobak2015group} extended the lost cow problem to $k$ cows (termed Mobile Entities) and showed that, independent of the number of cows, the group search cannot be better than the LostCow algorithm's performance. Czyzowicz et al. \cite{czyzowicz2019search} consider the problem of parallel search of a motionless target at an unknown location on a line, by $n$ robots. At most $f$ of these robots are considered to be faulty, and the remaining robots are reliable.
% \PRT{maybe add a line telling what the result for this last problem is?}

Two-player rendezvous search problem on the line has been well studied for both the
symmetric~\cite{anderson1995rendezvous,
baston1999note, han2008improved,
stachowiak2009asynchronous, baston1998rendezvous, ozsoyeller2013symmetric} and the
asymmetric~\cite{alpern1995rendezvous, baston1998rendezvous,
gal1999rendezvous, alpern1999asymmetric, alpern2000pure} versions. The symmetric version can be categorized according to whether the initial distance between the players is known \cite{alpern1995rendezvous, anderson1995rendezvous, baston1999note, han2008improved, stachowiak2009asynchronous} or unknown \cite{baston1998rendezvous, ozsoyeller2013symmetric, czyzowicz2018linear}. Baston and Gal \cite{baston1998rendezvous} considered the case where the initial distance $2d$ between the players is drawn from an unknown distribution with only  its expected value known to be $E[2d] = \mu$. They provided a symmetric
algorithm with expected meeting time $13.325\mu$ and competitive ratio $26.650$. In our previous work \cite{ozsoyeller2013symmetric}, we improved this result by providing a 17.686 distance-competitive and 24.843 time-competitive symmetric rendezvous algorithm for two robots at an unknown initial distance. In this paper, we extend this study to $n > 2$ robots.

The solutions provided for the multi-player linear rendezvous search problem in the literature assume that the distance between each pair of adjacent robots is known; in contrast, we consider it to be unknown. Lim et al. \cite{lim1997rendezvous} studied the rendezvous of $m \leq n$ blind, unit speed players. 
% The players are placed with a random permutation onto the integers 1 to $n$ on the line initially pointing randomly to the right or left. 
The authors showed that $R^{a}_{3, 2}$ is 47/48 and $R^{s}_{n, n}$ is asymptotic to $n/2$, where $R^{a}_{n, m}$ and $R^{s}_{n, m}$ denote the least expected rendezvous time of $m$ players for the asymmetric and symmetric strategy, respectively. Prior to this study, Lim and Alpern focused on the asymmetric version of the same problem and minimizing the maximum time to rendezvous rather than the expected time \cite{lim1996minimax}. The asymmetric value of the $n$-player minimax rendezvous time $M_n$ has an upper bound $n/2 + (n/ \log n) + o(n/ \log n)$. Gal \cite{gal1999rendezvous} presented a simpler strategy for the problem in \cite{lim1996minimax} and showed that the worst case meeting time has an asymptotic behavior of $n/2 + O(\log n)$. The rendezvous of only three agents is considered in \cite{baston1999note, alpern2002rendezvous}. Baston \cite{baston1999note} proposed a strategy that has a maximum rendezvous time of 3.5. Alpern and Lim \cite{alpern2002rendezvous} showed that Baston's strategy is the only strategy to achieve this result.

The deterministic rendezvous of two asynchronous agents on the line was studied in \cite{marco2006asynchronous, stachowiak2009asynchronous, czyzowicz2018linear}. To break the symmetry between the agents, \cite{marco2006asynchronous} considered that each agent has a distinct identifier, called label. A label is a nonempty binary string which is known to the agent. The cost of the proposed algorithm is $O(D|L_{min}|^{2})$ when $D$ is known and $O((D+|L_{max}|)^3)$ when $D$ is unknown, where $D$ is the initial distance between the agents, and $|L_{min}|$ and $|L_{max}|$ denote the lengths of the shorter and longer label of the agents, respectively. This bound was later improved by Stachowiak \cite{stachowiak2009asynchronous}.

The rendezvous search problem was also studied in other types of environments such as a circle \cite{georgiou2019symmetric, dessmark2006deterministic}, a graph \cite{marco2006asynchronous, thomas2001many, miller2016time, czyzowicz2012meet, dessmark2006deterministic, czyzowicz2012forget}, and a plane \cite{collins2010tell, bouchard2019asynchronous, anderson2001two, alpern2006rendezvous, bampas2010almost, ozsoyeller2019rendezvous}.

A generalization of the rendezvous search problem is the \emph{Gathering} problem \cite{flocchini2005gathering, klasing2010taking, cieliebak2002gathering, cieliebak2003solving, prencipe2005feasibility, jurek2009gathering, agmon2006fault, flocchini2013rendezvous, gordon2004gathering}. Gathering requires two or more robots in an environment to co-locate at one point in finite time. The robots operate in Wait-Look-Compute-Move cycles in which they decide on their moves viewing their surroundings and analyzing the configuration of robot locations. In contrast, rendezvous search assumes extremely limited sensing capabilities, for example only collision detection, hence the solution methods do not involve computing the positions of the other robots in the environment. 

\section{Problem Formulation}
\label{sec:formulation}

We consider $n > 2$ identical robots that are placed arbitrarily on a line. The initial distance between every pair of robots is unknown. Moreover, the robots are unaware of their own initial positions and the initial positions of the others. We assume that each robot knows the number $n$.
Rendezvous can occur anywhere on the line, since no rendezvous location is fixed in advance of the search. We study the synchronous case in which we assume that the robots have synchronized clocks, thus, they start executing the algorithm at the same time.

Let $x_{j}$ denote the initial position of robot $R_j$ in order from left to right, where $j \in {1, ..., n}$. Let $d$ be initial distance between the leftmost robot $R_1$ and the rightmost robot $R_n$. Without loss of generality, $R_1$ is located at $x = 0$ and $R_n$ is located at $x = d$. Based on the initial configuration of the robots, $R_1$ and $R_n$ are called the boundary robots and the robots that are located between $R_1$ and $R_n$ are called the internal robots. It is easy to see that each boundary robot has one neighbor while each internal robot has two neighbors.

We focus on minimizing the distance competitive
ratio of our algorithm. Due to using symmetric strategies, the expected distance traveled is identical for all robots. Here the expectation is taken over the randomization in the strategy and over the number of robots. Therefore, we only analyze it for one robot.
\section{Multi-robot Symmetric Rendezvous Algorithm}
\label{sec:msralgorithm}

In this section, we present the multi-robot symmetric rendezvous algorithm $\MSR$. Since the search is on the line, we maintain the invariant that the robots never cross paths. This implies that once the boundary robots ($R_1$ and $R_n$) meet, then the algorithm terminates, since all robots would meet. Therefore, the main idea in $\MSR$ is to find out if any robot pair can act as boundary robots and meet the rest of the robots to achieve rendezvous.

We assign a non-negative sequence $f_{-1}, f_0, f_1, f_2, \ldots$ to each robot, where $f_{-1}=0$ and $f_i = r^{i}$, for $i \geq 0$. Robots use the same expansion radius $r > 1$ which is fixed throughout the algorithm. We derive the optimized value of $r$ in the proof of Theorem~\ref{theo:MSR}. The algorithm proceeds in rounds indexed by integers $i \geq 0$. Each round consists of two consecutive phases: phase-1 and phase-2. The first direction the robot performs the search in round $i$ refers to phase-1. We denote this direction by $\dir$. In phase-2 of round $i$, the robot searches the opposite direction of $\dir$, which is denoted by $\notdir$. Let the initial distance between $R_1$ and $R_n$ be $d = r^{k+\delta}$, where $\delta\in(0,1]$ and $k\in\mathbb{Z^+}$. Note that $\delta$ and $k$ are unknown to the algorithm.

Our online algorithm is a hybrid of two search modes: randomized search and deterministic search. The robot executing $\MSR$ starts by performing a randomized search to break the symmetry but switches to a deterministic search once it meets a robot. A robot is called \emph{single} if it has not met any other robot yet. Thus, at the beginning of the algorithm, all robots are single. As long as the robot is single, it executes the randomized search mode in which it flips a coin at the beginning of each round to decide its itinerary. If it tosses heads, then it moves right in phase-1 and left in phase-2; otherwise it moves left in phase-1 and right in phase-2. Rendezvous can never occur if the coin flips of all robots in each round after the start come up the same. This also implies that the robots will always be executing the randomized mode of the algorithm. The randomized mode adopts the symmetric rendezvous algorithm $SR$ presented in our previous work for two robots \cite{ozsoyeller2013symmetric}.

When the single robot meets another single robot, it switches to the deterministic search mode. In this mode, the robot does not randomize its direction at the beginning of a round. Instead, it always starts the round in its deterministic direction $\dird$. If the single robot meets another single robot in phase-1 of round $i$, then $\dird = \notdir$. However, if the single robot meets another single robot in phase-2 of round $i$, then $\dird = \dir$. Once the robot is in the deterministic search mode, its search becomes similar to the one in the \emph{LostCow} algorithm, with the difference that in our algorithm, the robot uses the first direction it moves in a round (phase-1) to search for any undiscovered mobile robots instead of a stationary object while it uses the second direction it moves (phase-2) to see whether the rest of the robots are gathered.

We now describe the motion pattern of robot $R_j$. In round $i$, $R_j$ starts at one of $x=x_j \pm f_{2i-1}$. If the robot is to first search the right direction in round $i$, then it moves right to $x=x_j+f_{2i}$ in phase-1 and left to $x=x_j-f_{2i+1}$ in phase-2. If the robot is to first search the left direction in round $i$, then it moves left to $x=x_j-f_{2i}$ in phase-1 and right to $x=x_j+f_{2i+1}$ in phase-2. Assuming that $R_j$ does not meet any robot during round $i$, its possible ending position is $x=x_j \pm f_{2i+1}$. This is also the initial position for round $i+1$. The possible itineraries of $R_j$ based on its initial position in round $i$ are shown in Fig.~\ref{motionpat}.

\begin{figure}
\centering
\psfrag{R1}[bl][c][0.8][0]{\small{$\dir(i) =$ right}}
\psfrag{R2}[bl][c][0.8][0]{\small{$\dir(i-1) =$ left}}
\psfrag{R3}[bl][c][0.8][0]{\small{$\dir(i) =$ right}}
\psfrag{R4}[bl][c][0.8][0]{\small{$\dir(i-1) =$ right}}
\psfrag{L1}[bl][c][0.8][0]{\small{$\dir(i) =$ left}}
\psfrag{L2}[bl][c][0.8][0]{\small{$\dir(i-1) =$ right}}
\psfrag{L3}[bl][c][0.8][0]{\small{$\dir(i) =$ left}}
\psfrag{L4}[bl][c][0.8][0]{\small{$\dir(i-1) =$ left}}
\psfrag{rj}[bl][c][0.8][0]{$R_{j}$}
\psfrag{I1}[bl][c][0.8][90]{\small{$x_j$}}
\psfrag{A}[bl][c][0.8][0]{\small{$x_j+f_{2i-1}$}}
\psfrag{B}[bl][c][0.8][0]{\small{$x_j+f_{2i}$}}
\psfrag{C}[bl][c][0.8][0]{\small{$x_j+f_{2i+1}$}}
\psfrag{-A}[bl][c][0.8][0]{\small{$x_j-f_{2i-1}$}}
\psfrag{-B}[bl][c][0.8][0]{\small{$x_j-f_{2i}$}}
\psfrag{-C}[bl][c][0.8][0]{\small{$x_j-f_{2i+1}$}}
\includegraphics[width=0.8\columnwidth]{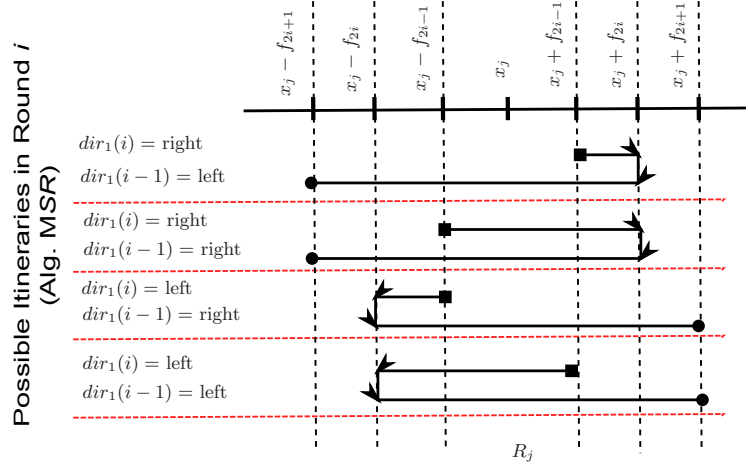} %[width=0.9\columnwidth]
\caption{The possible itineraries of robot $R_{j}$ executing algorithm $\MSR$. The arrows show the directions that the robot is moving in round $i$. The small square and circle icons show the beginning and end of a motion, respectively.}
\label{motionpat}
\end{figure}

In our algorithm, the robot can also be in one the following additional states beyond the initial \emph{single} state: $\boun$ and \emph{internal}. We determine these states based on the fact that an internal robot can meet another robot in both directions since it has two neighbors while a boundary robot can meet another robot in only one direction since it has only one neighbor. However, in $\MSR$, although the robot in \emph{internal} state is the actual internal robot in the initial configuration, the robot in $\boun$ state only acts as one but does not have to be the actual boundary robot. Thus, we refer $\boun$ as temporary boundary.

Next, we introduce the behaviours of the robot that meets another robot in a round. When single robots $R_j$ and $R_{j'}$ meet in round $i$, they both change their states to $\boun$. Recall that this is the last round before both robots switch to the deterministic search mode. Starting from round $i+1$, the $\boun$ robot $R_j$ no longer flips a coin but always starts a round in $\dird$. The robot in this state continues the search until it meets another $\boun$ robot in $\dird$. When this happens, it changes its state to internal and waits until it sticks to a $\boun$ robot. A sequence of consecutive internal robots is always surrounded by $\boun$ robots which carry out the search in both directions. This also implies that both neighbors of an internal robot can be internal, or one of its neighbors can be internal while the other is $\boun$. Thus, a single robot can never encounter an internal robot. It is easy to see that the robots $R_1$ and $R_n$ can never be internal since they each have only one neighbor, hence can meet the other only in one direction. A $\boun$ robot collects any single robot moving towards it. In other words, the single robot always sticks to and moves along with the $\boun$ robot that it meets. Therefore, the neighbors of a robot at the initial configuration can change when the robot changes its state to $\boun$. For example, consider the consecutive robots $R_{j-1}$, $R_{j}$, and $R_{j+1}$ at the beginning. Although $R_{j+1}$ is not initially adjacent to $R_{j-1}$, if $R_{j}$ sticks to $R_{j+1}$, then $R_{j+1}$ becomes adjacent to $R_{j-1}$. The pseudocode of $\MSR$ is presented in Algorithm~\ref{alg:msr} and Algorithm~\ref{alg:msrfunc}, and its finite state representation is shown in Fig. \ref{fig:fsm}.

\begin{figure}
\centering
\includegraphics[width=\columnwidth]{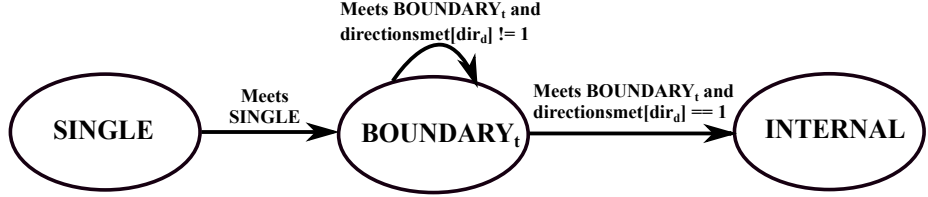}
\caption{Finite state machine representation of Algorithm $\MSR$.}
\label{fig:fsm}
\end{figure}

Let $D_{i,1}(D_{i,2})$ denote the distance traveled in phase-1(phase-2) of round $i$. If the direction in phase-1 of this round differs from the previous round, then $D_{i,1} = f_{2i}-f_{2i-1}$. Otherwise, $D_{i,1} = f_{2i}+f_{2i-1}$. Regardless of the direction in phase-1, $D_{i,2} = f_{2i}+f_{2i+1}$. Thus, the distance traveled (the length of an itinerary) in round $i$ is either $D_i = f_{2i+1} + 2f_{2i} - f_{2i-1}$ or $D_i = f_{2i+1} + 2f_{2i} + f_{2i-1}$.
The maximum total time required for round $i$ is $T_{i} = T_{i,1} + T_{i,2}$, where $T_{i,1} = f_{2i}+f_{2i-1}$ and $T_{i,2} = f_{2i}+f_{2i+1}$.

In the synchronous case, robots start each phase of a round at the same time. This
is managed by introducing waiting times in $\MSR$ algorithm. The waiting
time in a phase is determined by the following cases: (1) The robot does not
meet another robot. (2) The robot meets another robot but continues moving
without waiting at the met location, which occurs when a \boun robot meets a single or an internal robot. (3) The robot meets another robot and
waits at the met location till the end of the phase, which occurs when two \boun or two single robots meet.
If case-1 or case-2 occurs, then $W_{i,1} = T_{i,1} - D_{i,1}$ and $W_{i,2} = T_{i,2} - D_{i,2}$.
Let $D_{i,1,m}$($D_{i,2,m}$) denote the distance traveled by the robot until it meets another robot in phase-1(phase-2) of round $i$. Since any robot pair meets before they reach their ending positions in the phase, we have $D_{i,1,m}(D_{i,2,m}) < D_{i,1}(D_{i,2})$. If case-3 occurs, then $W_{i,1} = T_{i,1} - D_{i,1,m}$ and $W_{i,2} = T_{i,2} - D_{i,2,m}$.

\DontPrintSemicolon
\LinesNumbered

\begin{algorithm}
\algsetup{linenosize=\small}
%\scriptsize
\SetAlgoLined
\caption{$\MSR$, Multi-robot symmetric rendezvous search algorithm.}
\label{alg:msr}
%\begin{algorithmic}[1]
{\textbf{global}\;
\tcp{$\mathcal{S}$: SINGLE; $\mathcal{B}$: BOUNDARY\textsubscript{t}}
\tcp{$\mathcal{I}$: INTERNAL; $\mathcal{N}$: NOROBOT}
\quad $state = \mathcal{S}$\;
\quad $meets = \mathcal{N}$\;
\quad $stick = 0$ \;
\quad $firstmeeting = 0$\;
\quad $direction = 1$\;
\quad $directionsmet[2] = \{0\}$\;
\quad RIGHT $= 1$, LEFT $= -1$\;
\quad HEAD $= 0$, TAIL $= 1$\;
\textbf{end global}\;
\;
}
\tcp{$i$ denotes the current round.}
$i = 0$\;
$phase = 1$\;
$continues =$ \textbf{\upshape true}\;
\While{{\upshape Rendezvous()} $\neq$ \textbf{\upshape true}}{
    \If{$state ==$ {\upshape $\mathcal{S}$}}{
        $cointoss =$ Random(HEAD, TAIL)\;
        $direction =$ RIGHT\; %\tcp*{first move right}
        \If{$cointoss ==$ {\upshape TAIL}}{
        		$direction =$ LEFT\; %\tcp*{first move left}
        }
    }
    \If{$stick \neq 1$}{
		\If{$state \neq$ {\upshape $\mathcal{I}$}}{
			$phase = 1$\;
            $continues =$ MoveTo($direction * f(2i)$, $direction$, $phase$, $i$)\;
            \If{$continues ==$ \textbf{\upshape true}}{
            	$phase = 2$\;
                $continues =$ MoveTo($-1 * direction * f(2i+1)$, $-1*direction$, $phase$, $i$)\;
                \If{$continues ==$ \textbf{\upshape true}}{
                    $i = i + 1$\;
                }
            }
        }
        \If{$state ==$ {\upshape $\mathcal{I}$}}{
        		WaitInternal()\;
        }
    }

    \If{$stick == 1$}{
	   StickToRobot()\;
    }
}
\end{algorithm}

\begin{algorithm}
\algsetup{linenosize=\tiny}
%\scriptsize
\SetAlgoLined
\caption{MoveTo(\emph{destination}, \emph{currdirection}, \emph{phase}, \emph{round})}
\label{alg:msrfunc}
    $dir = 0$\;
    $prevpos =$ GetCurrentPos()\;
    \;
    \If{$currdirection ==$ {\upshape LEFT}}{
        $dir = 1$ \;
    }
    \If{$firstmeeting == 1$}{
        $direction = -1 * currdirection$\;
        $firstmeeting =$ 0\;
    }
    \;
    \While{{\upshape Reach($destination$)} $\neq$ \textbf{\upshape true}} {
         $meets =$ CheckState()\;

        \uIf{($state ==$ {\upshape $\mathcal{S}$} \textbf{\upshape or} $state ==$ {\upshape $\mathcal{B}$}) \textbf{\upshape and} $meets ==$ {\upshape  $\mathcal{N}$}}{
             Move($currdirection$) \;
        }
        \uElseIf{$state ==$ {\upshape $\mathcal{S}$} \textbf{\upshape and} $meets ==$ {\upshape $\mathcal{B}$}} {
            $directionsmet[dir] = 1$ \;
            $stick = 1$ \;
            \KwRet{\textbf{false}} \;
        }
        \uElseIf{$state ==$ {\upshape $\mathcal{S}$} \textbf{\upshape and} $meets ==$ {\upshape $\mathcal{S}$}}{
        	$state = \mathcal{B}$\;
            $directionsmet[dir] = 1$\;
            $firstmeeting = 1$\;
            \textbf{break}\;
        }
        \uElseIf{$state ==$ {\upshape $\mathcal{B}$} \textbf{\upshape and} ($meets ==$ {\upshape $\mathcal{I}$} \textbf{\upshape or} $meets ==$ {\upshape $\mathcal{S}$})}{
            StickToAndMove($currdirection$, $destination$)\;
        }
        \ElseIf{$state ==$ {\upshape $\mathcal{B}$} \textbf{\upshape and} $meets ==$ {\upshape $\mathcal{B}$}}
         {
            \uIf{$directionsmet[dir] \neq 1$}{
                 $directionsmet[dir] = 1$\;
                 $state = \mathcal{I}$\;
            	 \KwRet{\textbf{false}}\;
            }
            \uElseIf{$directionsmet[dir] == 1$ \textbf{\upshape and} {\upshape Total()} != n}{
               \textbf{break}\;
            }
            \ElseIf{$directionsmet[dir] == 1$ \textbf{\upshape and} {\upshape Total()} == n}{
                rendezvous = 1\;
                \KwRet{\textbf{false}} \;
            }
        }
    }
    \;
    $currpos =$ GetCurrentPos() \;
    $D_{round, phase} =$ DistanceTrav($prevpos$, $currpos$) \;
    $W_{round, phase} = T_{round, phase} - D_{round, phase}$. \;
    Wait($W_{round, phase}$) \;
     \;
    \KwRet{\textbf{\upshape true}} \;
\end{algorithm}

An example execution of Algorithm $\MSR$ when $d=50$,
$r=1.28$, and $n = 10$ is presented in Fig. \ref{fig:execution}. Each robot has unit speed. Fig. \ref{fig:execution} (Top) shows the positions of the robots  until rendezvous. The small circles on the plot depict the meeting time of the robots. The rendezvous occurs after 6 rounds. Fig. \ref{fig:execution} (Bottom) shows the robot meetings and the changes in their states during a round. Initially, all robots are in single (S) state. In phase-1 of round 2, $R_4$ and $R_5$ meet and change their states to boundary (B). This is followed, in the next phase, by the meetings of two pairs of single robots: $R_0$ and $R_1$, and $R_8$ and $R_9$, which also change their states to boundary. In round 3, single robots $R_4$, $R_6$, $R_8$ meet and stick to boundary robots $R_4$, $R_5$, and $R_7$, respectively. The ID of the robot that sticks to another robot is shown below this robot's ID. In round 4, single $R_2$ meets boundary $R_1$ and sticks to it. Next, we see that the boundary robots $R_1$ and $R_4$ meet in their deterministic direction and become internal robots. The boundary robots $R_5$ and $R_8$ also meet in their deterministic direction and become internal robots. In phase-2 of round 5, $R_0$ meets and collects the waiting internal robots $R_1$ and $R_4$, while $R_9$ meets and collects the waiting internal robots $R_8$ and $R_5$. Finally, the rendezvous occurs when robots $R_0$ and $R_9$ meet.

\begin{figure}[!ht]
\centering
\vspace{0.2cm}
\includegraphics[width=0.8\columnwidth]{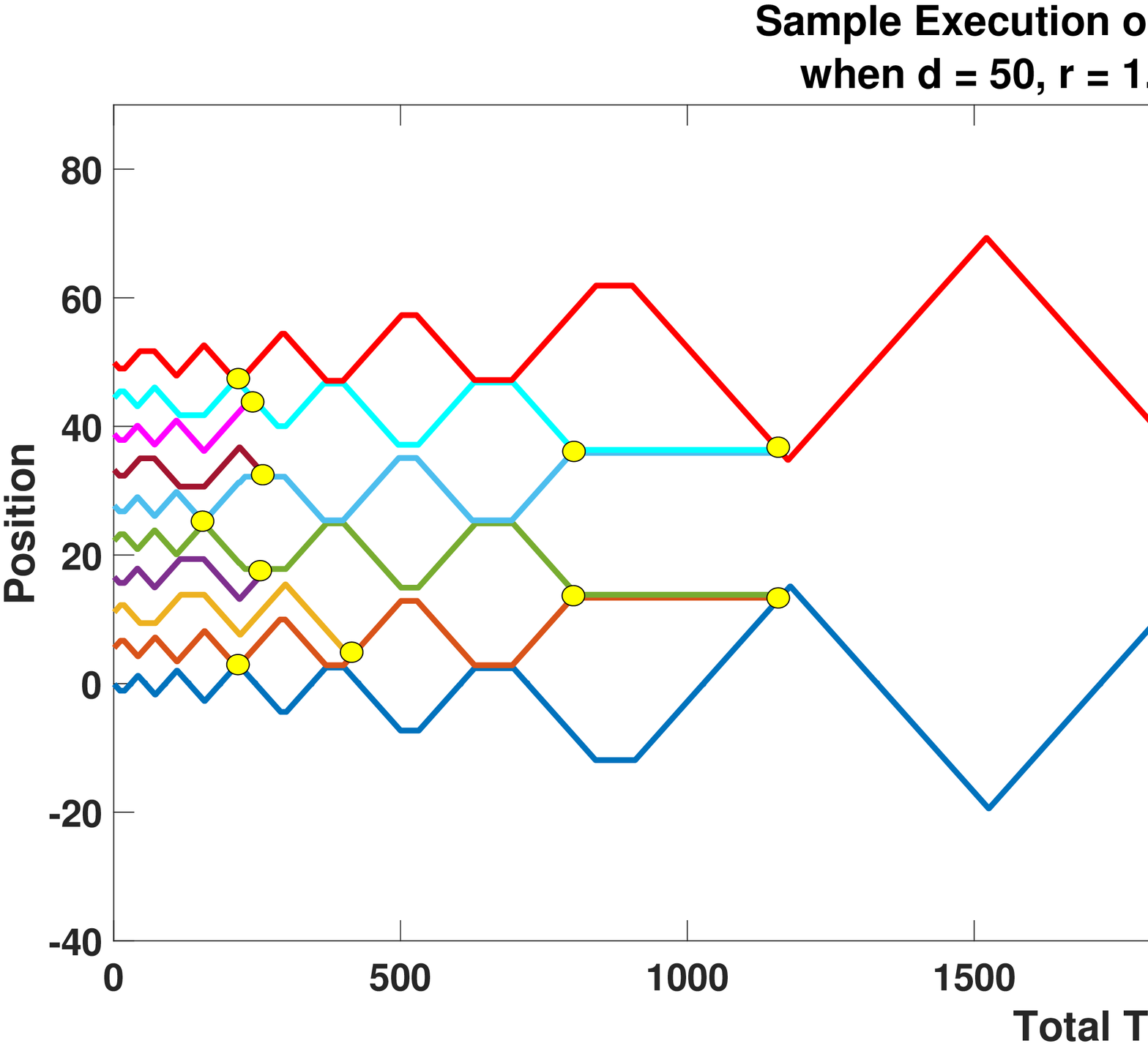}\\
\vspace{0.2cm}
\includegraphics[width=0.7\columnwidth]{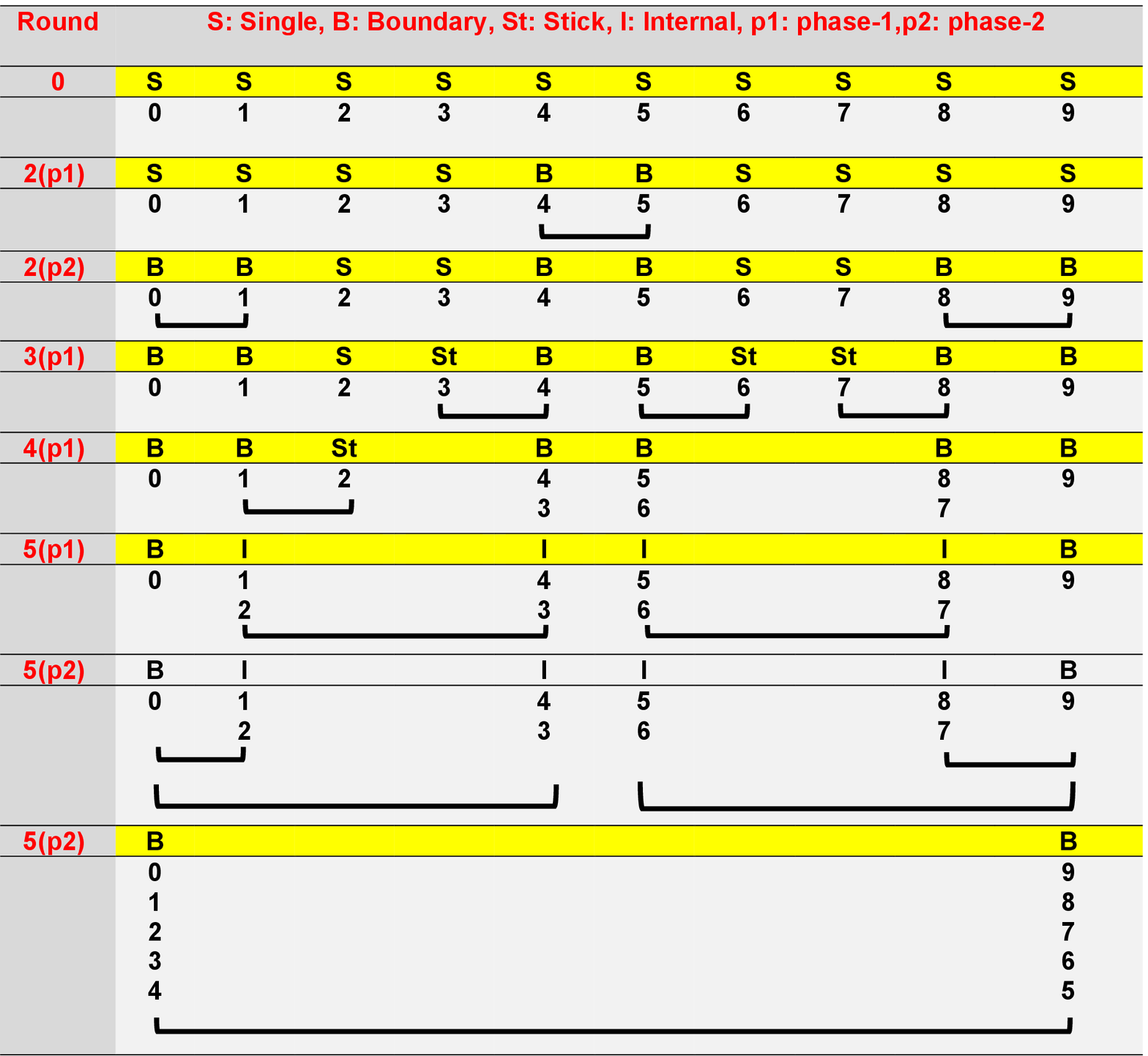}
\caption{Example Execution of Algorithm $\MSR$.}
\label{fig:execution}
\end{figure}
\section{Analysis of Algorithm $\MSR$}
\label{sec:msranalysis}

In this section, we find an upper bound on the expected distance traveled by the robot and compare it with the optimal offline solution. Note that since our strategy is symmetric, the expected distance traveled is identical for all robots. In the offline setting, the robots know their positions on the line and the initial distance between the boundary robots. Thus, the solution would be for them to move toward $x = d/2$. Our main result is presented in Theorem~\ref{theo:MSR}.

Let $\alpha$ be the first round that any robot can travel far enough to reach the meeting location of any pair of robot.
We divide the execution of the analysis into two stages. \emph{Stage-1} consists of rounds $0 \leq i \leq \alpha-1$. We ignore the possibility of  rendezvous occurring in this stage and bound the probability of rendezvous not occurring.  \emph{Stage-2} consists of rounds $i \geq \alpha$ in which the rendezvous can occur. In Lemma~\ref{lemma:rendezvous}, we show that if any pair of robots meet in Stage-2, then all robots meet. Rendezvous cannot occur in this stage as long as the coin flips of all robots come up the same, when all robots are still in the randomized search mode of $\MSR$. 
% We can see that the rendezvous behavior in this stage is consistent for all the robots. \PRT{not sure what consistent means?} 
Thus, we compute the expected distance traveled using an infinite sum.

Let $S_i$ denote the event that at least one adjacent single robot pair $p$ meets in round $i$. Here $p$ can be any adjacent robot pair. $S_i$ will be our proxy event for rendezvous. Let $M_{i}^{j}$ be the event that $R_{j}$ initially moves right and $\notM_{i}^{j}$ be the event that $R_{j}$ initially moves left in round $i$. The adjacent robots $R_j$ and $R_{j+1}$ can meet if event
$E_1 = M_{i}^{j} \wedge \notM_{i}^{j+1} \text{ or } E_2 = \notM_{i}^{j} \wedge M_{i}^{j+1}$
occurs. The probability that a single robot meets an adjacent single or \boun{} robot in round $i$ is $\Pee[E_1 \vee E_2] = 1/2$. On the other hand, $R_j$ and $R_{j+1}$ cannot meet when they initially move in the same direction; that is if event
$E_3 = M_{i}^{j} \wedge M_{i}^{j+1} \text{ or } E_4 = \notM_{i}^{j} \wedge \notM_{i}^{j+1}$
occurs. Thus, we have
\[ \Pee[S_{i}] =
  \begin{cases}
    0      & \quad \text{for } i \leq \alpha-1,\\
    \frac{2^{n}-2}{2^{n}} & \quad \text{for } i \geq \alpha.
  \end{cases}
\]

Next, we calculate the value of round $\alpha$ in terms of $k$ and $r$. To upper bound the expected distance traveled, we consider the furthest possible meeting location $x_m$ for $p$. Assuming that one of the robots in $p$ can be one of the actual boundary robots, either $x_m \in (d, d+f_{2i-1})$ or $x_m \in (-f_{2i-1}, 0)$. Let this boundary robot be $R_n$. Then, robot $R_j$ is able to reach $x_m$ if $x_j + f_{2i+1} \geq d+f_{2i-1}$. In the worst case, $R_j = R_1$, so $x_j = x_1 = 0$. Therefore, we have
%\begin{comment}
\begin{align}
f_{2i+1} \geq d + f_{2i-1},
\label{eq:phase1passesmax}
\end{align}
%\end{comment}
which holds for $\alpha \geq k/2+1-\log_{r^2}{r^{2}-1}$. Henceforth, $\alpha = k/2+i^{*}$, where $i^{*} = 1+\lceil\mid\log_{r^2}{r^{2}-1}\mid\rceil$.

Let $A_i$ denote the event that the algorithm is still active in round $i$.
We assume that Algorithm $\MSR$ is always active when $i < \alpha$, thus $\Pee[A_{i < \alpha}]=1$. In Lemma~\ref{lemma:rendezvous}, we show that rendezvous cannot occur in round $i \geq \alpha$ only if all robots' coin flips come up the same. Therefore, the probability that the algorithm is still active in round $i \geq \alpha$ is $\Pee[A_{i \geq \alpha}] = \Pee[\notS_0 \wedge \cdots \wedge \notS_{i-1}]$. Thus, we have
\begin{equation}
\label{eqn:probphase}
\Pee[A_i] = \left\{
\begin{array}{cc}
1 & 0 \leq i \leq \alpha-1, \\
\left(\frac{2}{2^{n}}\right)^{i-\alpha+1} & \alpha \leq i.  %(1/2)^(i-k) i>k   i+k+1   i+k
\end{array}
\right.
\end{equation}

%Due to space limitations, we provide a brief sketch of the rest of our analysis. The proofs of all our lemmas and theorem can be found in \cite{ozsoyeller2021multi}.

%\begin{comment}
We continue by analyzing the distance traveled in Stage-1 and Stage-2.
% \subsection{\textbf{Analysis of Stage-1}}
% \label{subsec:stg1}
We start with the computation of the expected distance traveled during Stage-1 which encompasses the rounds $i < \alpha$.
%\end{comment}
\begin{lemma}
\label{lemma:stage1} The expected distance traveled during Stage-1 satisfies
\begin{align}
&\sum_{i=0}^{\alpha-1} \Ee[D_i \mid A_i] \Pee[A_i] < r^{k+2i^{*}}\left(\frac{r+2+r^{-1}}{r^2-1}\right).
\label{eq:stage1}
\end{align}
\end{lemma}
%\begin{comment}
\begin{proof}
Since we consider that rendezvous cannot occur in this stage, $\Pee[A_i] = 1$.
The distance traveled by a robot in round $i$ is $D_i \leq f_{2i+1} + 2f_{2i} + f_{2i-1}$. Therefore,
\begin{align*}
\sum_{i=0}^{\alpha-1} \Ee[D_i \mid A_i] \Pee[A_i] & < \sum_{i=0}^{\alpha-1} \left(f_{2i+1} + 2f_{2i} + f_{2i-1}\right)\cdot 1 \\
%&\qquad < \sum_{i=0}^{\alpha-1} (r^{2i+1} + 2r^{2i} + r^{2i-1})\\
%&\quad < \left(r+2+r^{-1}\right)\frac{r^{2(k/2+i^{*})} - 1}{r^2-1}\\
& < r^{k+2i^{*}}\left(\frac{r+2+r^{-1}}{r^2-1}\right) %\text{ for $r$ = 1.26.}
\end{align*}
\end{proof}
%\end{comment}

% \subsection{\textbf{Analysis of Stage-2}}
% \label{subsec:stg2}
%\begin{comment}
We now compute the expected distance traveled for all rounds $i \geq \alpha$. Unlike Stage-1, rendezvous occurs in this stage with nonzero probability.
In the next lemma, we show that the meeting of any pair of single robots always results in rendezvous in the next round.
%\end{comment}

%\begin{comment}
\begin{figure}
\centering
    \begin{subfigure}[c]{0.7\columnwidth}
        \includegraphics[width=\columnwidth]{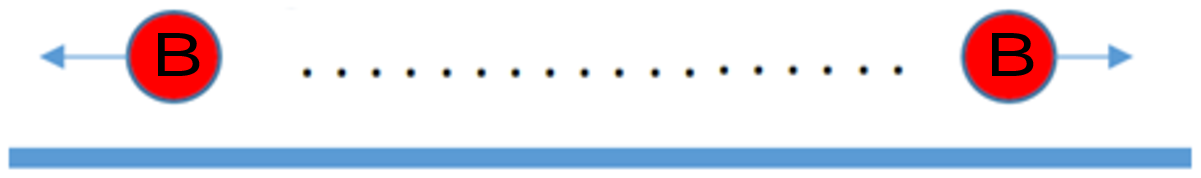} %[width=0.9\columnwidth]
        \caption{}
        \label{fig:lemmacase1}
    \end{subfigure}
    ~ %add desired spacing between images, e. g. ~, \quad, \qquad, \hfill etc.
      %(or a blank line to force the subfigure onto a new line)
    \begin{subfigure}[c]{0.7\columnwidth}
        \includegraphics[width=\columnwidth]{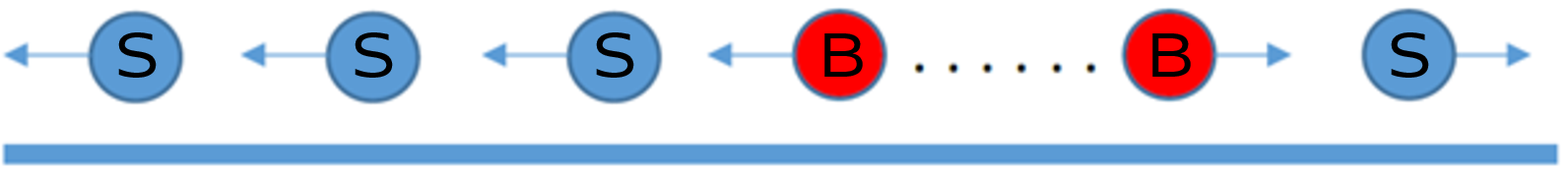} %[width=0.9\columnwidth]
        \caption{}
        \label{fig:lemmacase2}
    \end{subfigure}
\caption{Two possible configurations of the robots at the end of phase-1 of round $i+1$, given that event $S_i$ occurs in round $i$, that is at least one adjacent single robot pair $p$ meets in round $i$. B refers to the \boun{} robot and S refers to the single robot.}
\label{lemmacases}
\end{figure}
%\end{comment}

\begin{lemma}
Given that event $S_{i}$ occurs in round $i \geq \alpha$, then the robots always achieve rendezvous in round $i+1$.
\label{lemma:rendezvous}
\end{lemma}

%\begin{comment}
\begin{proof}
Consider that all robots are single at the beginning of round $i$. (\ref{eq:phase1passesmax}) ensures that any adjacent single robot pair can meet if event $E_1 \vee E_2$ occurs between them in round $i \geq \alpha$ and thus can become a \boun{} robot pair. The actual boundary robots can never be internal. Moreover, each sequence of consecutive internal robots is always located between two \boun{} robots which move in opposite deterministic directions. Therefore, it is guaranteed that once there is a \boun{} robot pair in round $i$, then there will always be at least one active \boun{} robot pair until rendezvous.

First, we show that the rendezvous may not occur in round $i$ if this is the first round that there is a \boun{} robot pair $p$. In $\MSR$, if pair $p$ is formed in phase-1, then both robots in $p$ wait till the end of this phase, and move away from each other in phase-2. Therefore, rendezvous cannot occur in this round.

Following an unsuccessful round $i$, the robots will start round $i+1$. Now, consider that we have at least one \boun{} robot pair $p$ at the beginning of round $i+1$. At the end of phase-1 of round $i+1$, the robots can have two possible configurations. In the first configuration, which is shown in Fig.~\ref{lemmacases}(a), the rest of the robots are between the leftmost \boun{} robot $R_j$ and the rightmost \boun{} robot $R_{j'}$. This can occur in two ways. The first way is when $R_j$ and $R_{j'}$ are the actual boundary robots. The second way is when $R_j$ and $R_{j'}$ are not the actual boundary robots but have already met any single robot in their deterministic direction. That is, any robot that started to the left of $R_j$ and right of $R_{j'}$ would have already met $R_j$ and $R_{j'}$.
Recall that a single robot sticks to the \boun{} robot that it meets. In phase-2 of the first configuration, $R_j$ and $R_{j'}$ move towards each other while collecting any robot between them, which would be already in internal state, thus the rendezvous occurs. 

In the second configuration shown in Fig.~\ref{lemmacases}(b), there can only be single robots in the deterministic directions of $R_j$ and $R_{j'}$. In this case, the single robots on the left side of $R_j$ move in the same direction as $R_j$, while the single robots on the right side of $R_{j'}$ move in the same direction as $R_{j'}$. Thus, the rendezvous cannot occur in phase-1. However, in phase-2, $R_j$ and $R_{j'}$ move towards each other while collecting any robots between them and start waiting before starting the next round. Meanwhile, the single robots move toward the waiting \boun{} robots $R_j$ and $R_{j'}$ and meet them, thus the rendezvous occurs. (\ref{eq:phase1passesmax}) also ensures that while the \boun{} robots are waiting for the end of phase-2 at their meeting location, a single robot moving towards this location would be able to meet them.
\end{proof}
%\end{comment}

\begin{lemma}
\label{lemma:stage2} The expected distance traveled during Stage-2 satisfies
\begin{align}
&\sum_{i=\alpha}^{\infty} \bigg[(\Ee[ D_i \mid \notS_i]\Pee[\notS_i]) + (\Ee [ D_i \mid S_i ]\Pee[S_i])\bigg]\Pee[A_i]\notag\\
&\qquad <  r^{k+2i^{*}}\left(\frac{0.75r^3+1.5r^2+1.75r+r^{-1}+2}{4-r^2}\right).
\label{eq:stage2}
\end{align}
\end{lemma}

%\begin{comment}
\begin{proof}
If $S_i$ does not occur in round $i$, then the expected distance traveled in round $i$ is
\begin{align}
&\Ee[D_{i} \mid \notS_{i}] \leq f_{2i+1} + 2f_{2i} + f_{2i-1}.
\label{eq:stage2Dunsuccess}
\end{align}
By Lemma~\ref{lemma:rendezvous}, event $S_i$ in round $i$ results in rendezvous in round $i+1$. Therefore, the distance traveled when $S_i$ occurs is equal to the distance traveled in rounds $i$ and $i+1$. Hence, we have
\begin{align}
\Ee[D_{i} \mid S_{i}] &< \left(f_{2i+1} + 2f_{2i} + f_{2i-1}\right) + \notag\\
&\quad\left(f_{2(i+1)+1} + 2f_{2(i+1)} + f_{2(i+1)-1}\right)\notag\\
& < r^{2i}\left(r^3+2r^2+2r+r^{-1}+2\right).
\label{eq:stage2Dsuccess}
\end{align}

The expected distance traveled during Stage-2 is the sum of the expected distance traveled in unsuccessful
rounds and the expected distance traveled in the (final) successful round, which is given by
\begin{align}
\sum_{i=\alpha}^{\infty} \bigg[(\Ee[ D_i \mid \notS_i]\Pee[\notS_i]) + (\Ee [ D_i \mid S_i ]\Pee[S_i])\bigg]\Pee[A_i].
\label{eq:s2uandsbound}
\end{align}
First, we compute the expected distance traveled in successful rounds during Stage-2 using (\ref{eqn:probphase}) and (\ref{eq:stage2Dsuccess}), which is bounded by
\begin{align}
& \sum_{i=\alpha}^{\infty}\left(\Ee [ D_i \mid S_i ]\Pee[S_i]\right)\Pee[A_i]\notag\\
%& \quad < \sum_{i=\alpha}^{\infty} \Ee [ D_i \mid S_i ] \left(\frac{2^n-2}{2^n}\right)\left(\frac{2}{2^{n}}\right)^{i-\alpha+1} \notag\\
%& \quad < r^{k+2i^{*}}\left(r^3+2r^2+2r+r^{-1}+2\right)\left(\frac{2^n-2}{2^{n}}\right)\left(\frac{2}{2^{n}}\right)\sum_{i=0}^{\infty}r^{2i}\left(\frac{2}{2^{n}}\right)^{i}\notag\\
%& \quad < r^{k+2i^{*}}\left(r^3+2r^2+2r+r^{-1}+2\right)\frac{2^n-2}{2^{n-1}\left(2^n-2r^2\right)}\notag\\
& \quad < 0.75r^{k+2i^{*}}\left(\frac{r^3+2r^2+2r+r^{-1}+2}{4-r^2}\right) \text{ for $n > 2$.}
\label{eq:s2sbound}
\end{align}
Next, we compute the expected total distance traveled during unsuccessful Stage-2 rounds using (\ref{eqn:probphase}) and (\ref{eq:stage2Dunsuccess}), which is bounded by
\begin{align}
& \sum_{i=\alpha}^{\infty} \left(\Ee[ D_i \mid \notS_i]\Pee[\notS_i]\right)\Pee[A_i]\notag\\
%& \quad < \sum_{i=\alpha}^{\infty} r^{2i}\left(r + 2 + r^{-1}\right) \left(\frac{2}{2^n}\right)\left(\frac{2}{2^{n}}\right)^{i-\alpha+1} \notag\\
%& \quad < r^{k+2i^{*}}\left(r + 2 + r^{-1}\right)\left(\frac{4}{2^{2n}}\right)\sum_{i=0}^{\infty}r^{2i}\left(\frac{2}{2^{n}}\right)^{i}\notag\\
%& \quad < r^{k+2i^{*}}\left(r + 2 + r^{-1}\right)\left(\frac{2}{2^{n-1}\left(2^n-2r^2\right)}\right)\notag\\
& \quad < 0.25r^{k+2i^{*}}\left(\frac{r + 2 + r^{-1}}{4-r^2}\right) \text{ for $n > 2$.}
\label{eq:s2ubound}
\end{align}
Therefore, we conclude that (\ref{eq:s2uandsbound}) is bounded by the sum of (\ref{eq:s2sbound}) and (\ref{eq:s2ubound}).
\end{proof}
%\end{comment}

\begin{theorem}
For the choice of $r = 1.28$, $\MSR$ algorithm has the competitive ratio of 54.732 and the competitive ratio of 34.154 as $n\to\infty$.
\label{theo:MSR}
\end{theorem}
%\begin{comment}
\begin{proof}
The expected distance traveled is obtained by adding the expressions in equations (\ref{eq:stage1}) and (\ref{eq:stage2}). Since the initial distance between the leftmost and rightmost robot on the line is $d = r^{k+\delta}$, we first replace each occurrence of $r^{k}$ with $dr^{-\delta}$ and $n$ with 3. Then, we divide it by $d/2$, which is the cost of the optimal offline solution. This expression is maximized at $\delta = 0$, and the coefficient of $d$ in it is minimized when $r = 1.28$. As a result, we guarantee the competitive ratio of 54.732. As $n\to\infty$, the distance traveled in Stage-2 approaches to 0; thus the competitive ratio of our algorithm approaches to 34.154.
\end{proof} 
%\end{comment}
\section{Simulations}
\label{sec:simulations}

In this section, we study the performance of our algorithm in simulations and verify the bound obtained in Theorem \ref{theo:MSR}. 
We implemented a multithreaded system using Java, where each robot is represented as a separate thread that executes the Algorithm $\MSR$. Varying the parameters $d$, $n$, and $r$, we report the results of the following: average distance competitive ratio, average number of rounds, average total time, average time competitive ratio, and average distance traveled. Each result is averaged over 100 trials, where each trial uses the maximum of the respective data among $n$ robots.

Fig. \ref{fig:simulationsgrp1}(Left) and Fig. \ref{fig:simulationsgrp1}(Right) show the average distance competitive ratio as a function of $n$ an $d$ and as a function of $n$ and $r$, respectively. This ratio is obtained by dividing the maximum distance traveled by $d/2$, where $d$ is the respective initial distance value tested in the simulations. In the left plot, $n$ is varied between 4 and 64 while $d$ is varied between $50$ and $125$. The robots are distributed uniformly random on the line. We observe that the average distance competitive ratio is higher when $n$ is small and stays constant as $n$ and $d$ increases. We further see that this ratio is smaller than our theoretically proved upper bound. Recall that the optimal $r$ value obtained in Theorem \ref{theo:MSR} is 1.28. In the right plot, this value is shown with double stars. With respect to the change in $n$ and when $d$ is fixed, we can see that $\MSR$ performs the best for $r = 1.28$.

Next, in Fig. \ref{fig:simulationsgrp2}, we investigate the average number of rounds to achieve rendezvous with respect to the change in $n$ and $d$ (Left plot), also with respect to the change in $n$ and $r$ (Right plot). The average of total number of rounds stays constant as $n$ increases and decreases as $r$ increases. Since the number of rounds is proportional to the distance the traveled by the robots, in the left plot, we see that the average number of rounds increases as $d$ increases, but stays constant with respect to the change in $n$.

The total time required for rendezvous is equal to the sum of the total distance traveled and total waiting time. Figure \ref{fig:simulationsgrp3} shows the results of the simulations for the average total time with respect to the change in $n$. Left plot reports the results for various $d$ when $r = 1.28$, whereas the right plot reports the results for various $r$ values when $d = 75$. As expected, the total time increases as $d$ increases and stays constant as $n$ increases. In the right plot, we observe that $r = 1.28$ outperforms the other $r$ values as $n$ increases. Figure \ref{fig:simulationsgrp4}(Left) shows the time competitive ratio with respect to the change in $n$ and $d$. Comparing the distance and time ratios in Figures \ref{fig:simulationsgrp1} and \ref{fig:simulationsgrp4}, we see that the distance competitive ratio is less than time competitive ratio. 

Fig. \ref{fig:simulationsgrp5} provides the average distance traveled for various $n$ values with respect to the change in $d$. As in the average total time results, we see that the average distance traveled increases as $d$ increases and stays constant as $n$ increases for $r = 1.28$. In the right plot, we compare the effect of equidistant and uniformly random initial placements of the robots on the performance of our algorithm. For various $n$ values and $d = 75$, the average distance traveled is smaller when the robots are randomly located on the line compared to the other case. The performances of $\MSR$ in these two cases approach each other as $n$ increases.

\begin{figure}[!ht]
\centering
\vspace{0.1cm}
\includegraphics[width=0.49\columnwidth]{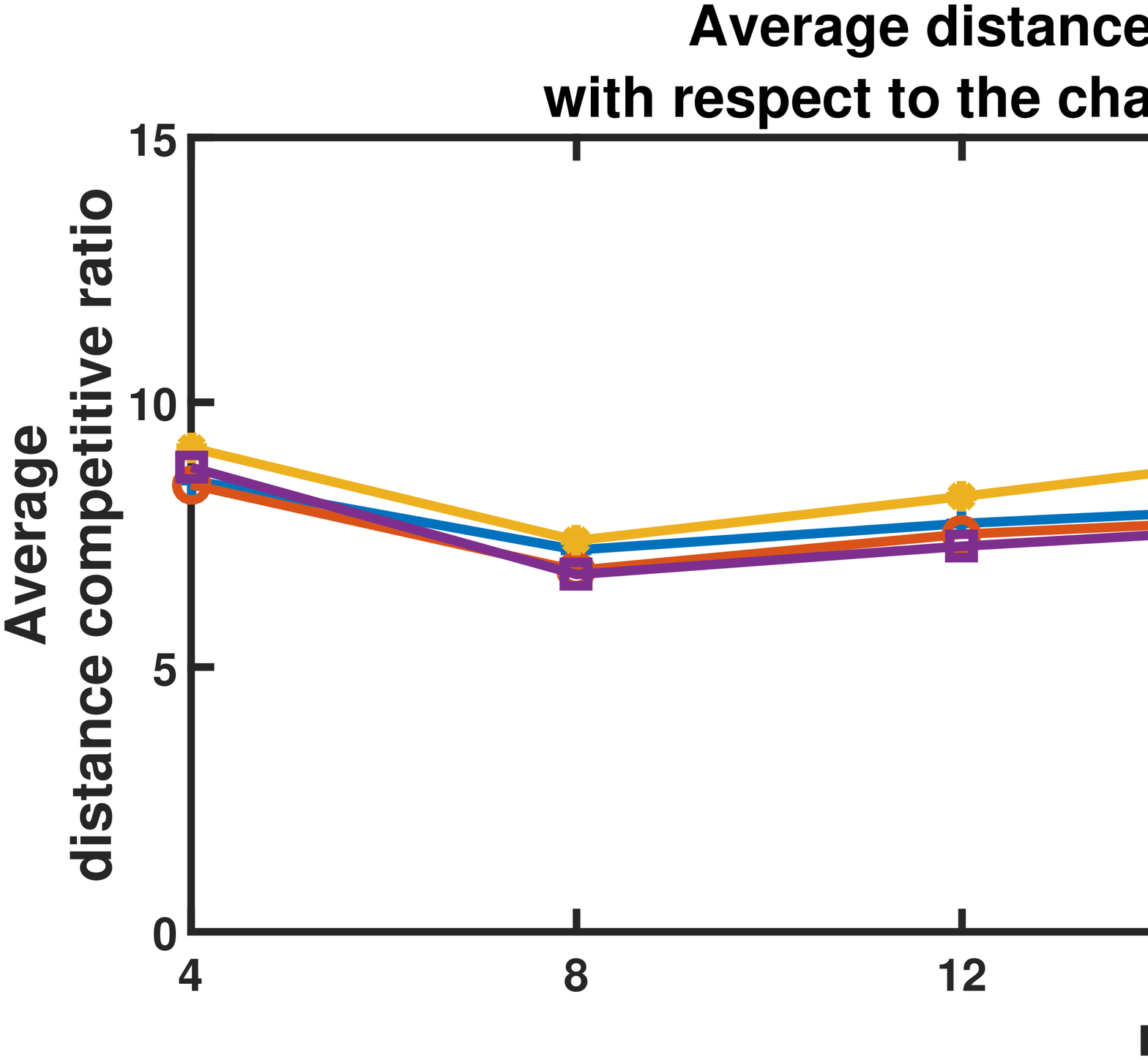}
\includegraphics[width=0.49\columnwidth]{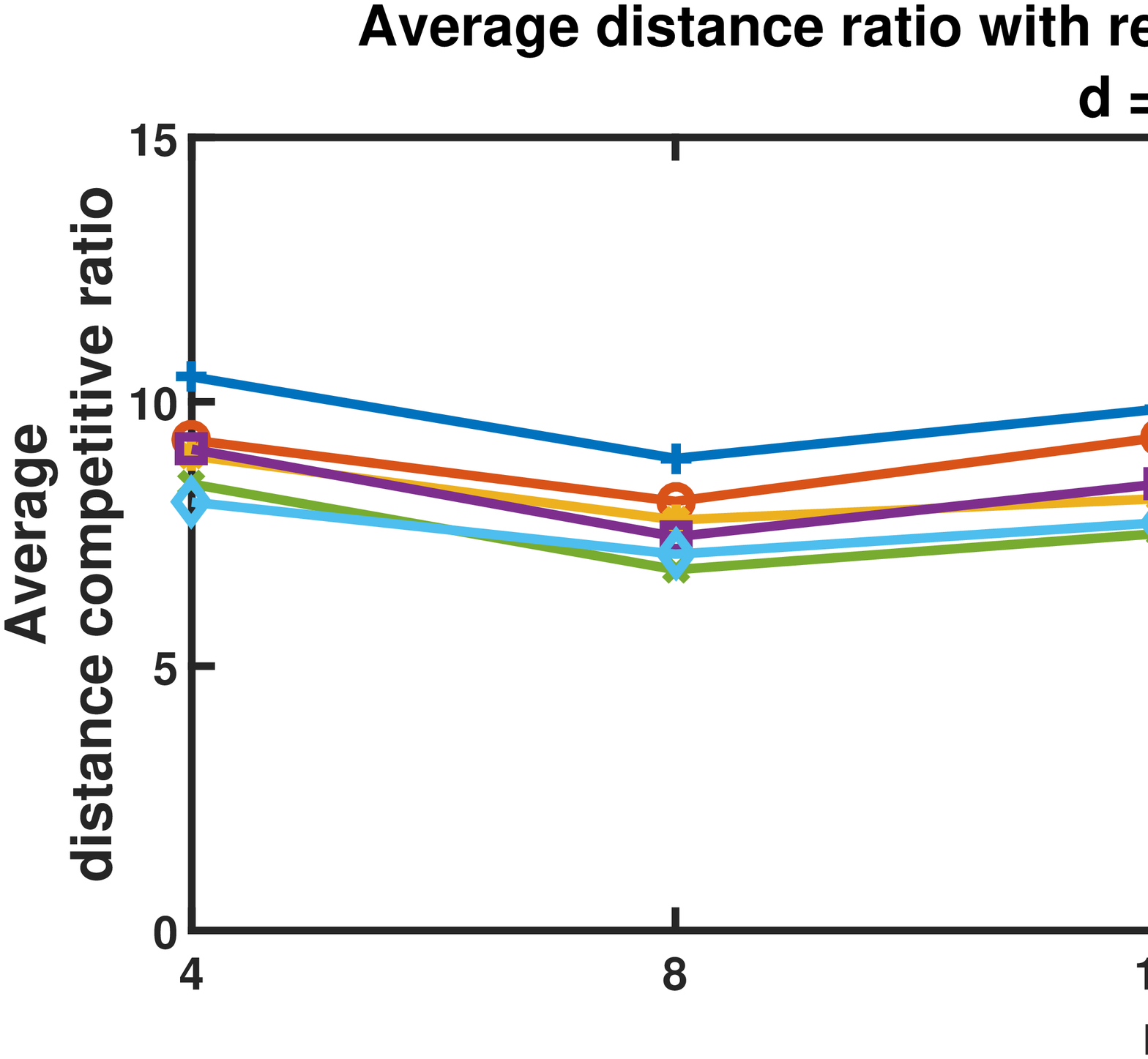}
\caption{Average distance competitive ratio.}
\label{fig:simulationsgrp1}
\end{figure}

\begin{figure}[!ht]
\centering
\includegraphics[width=0.49\columnwidth]{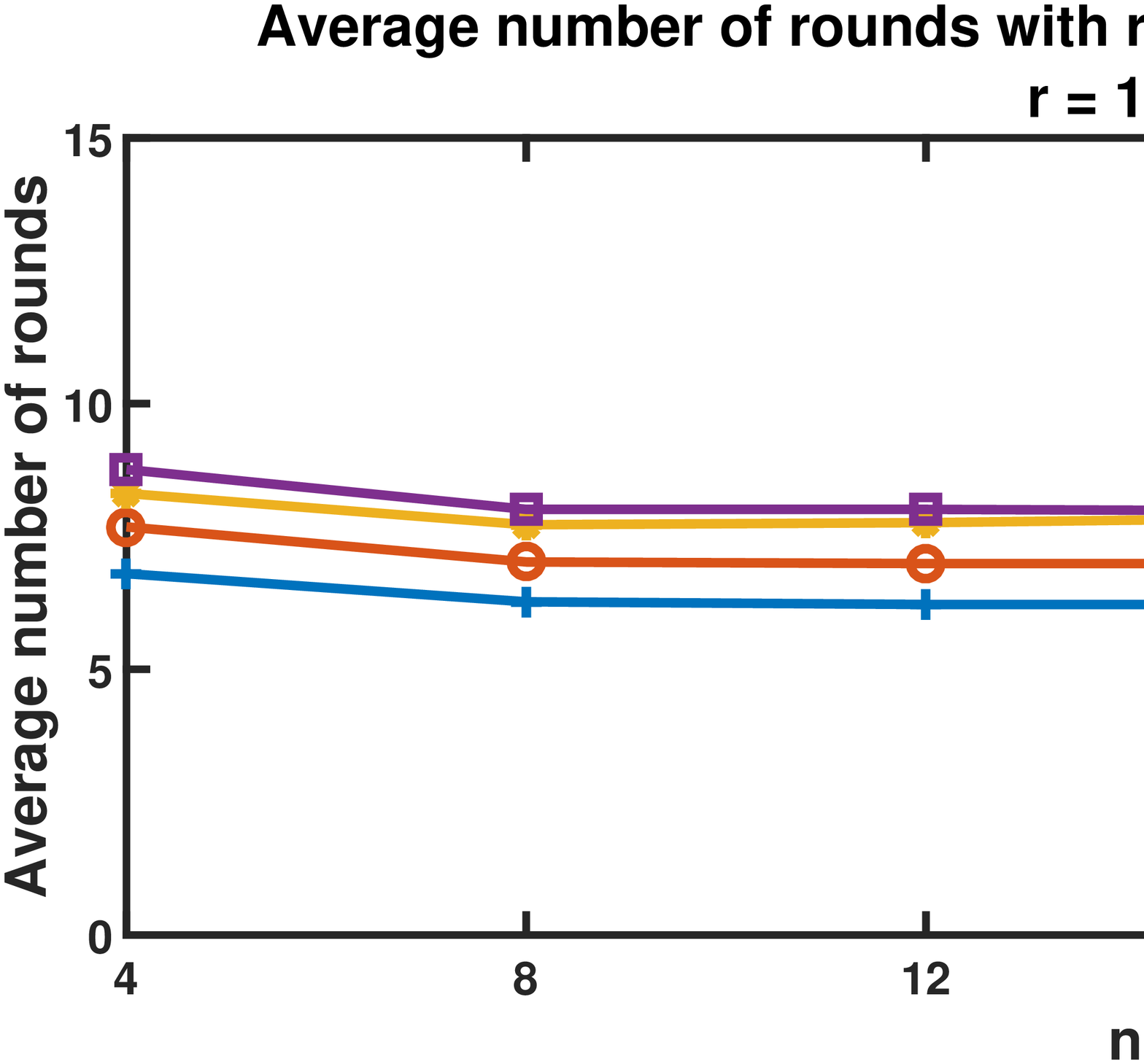}
\includegraphics[width=0.49\columnwidth]{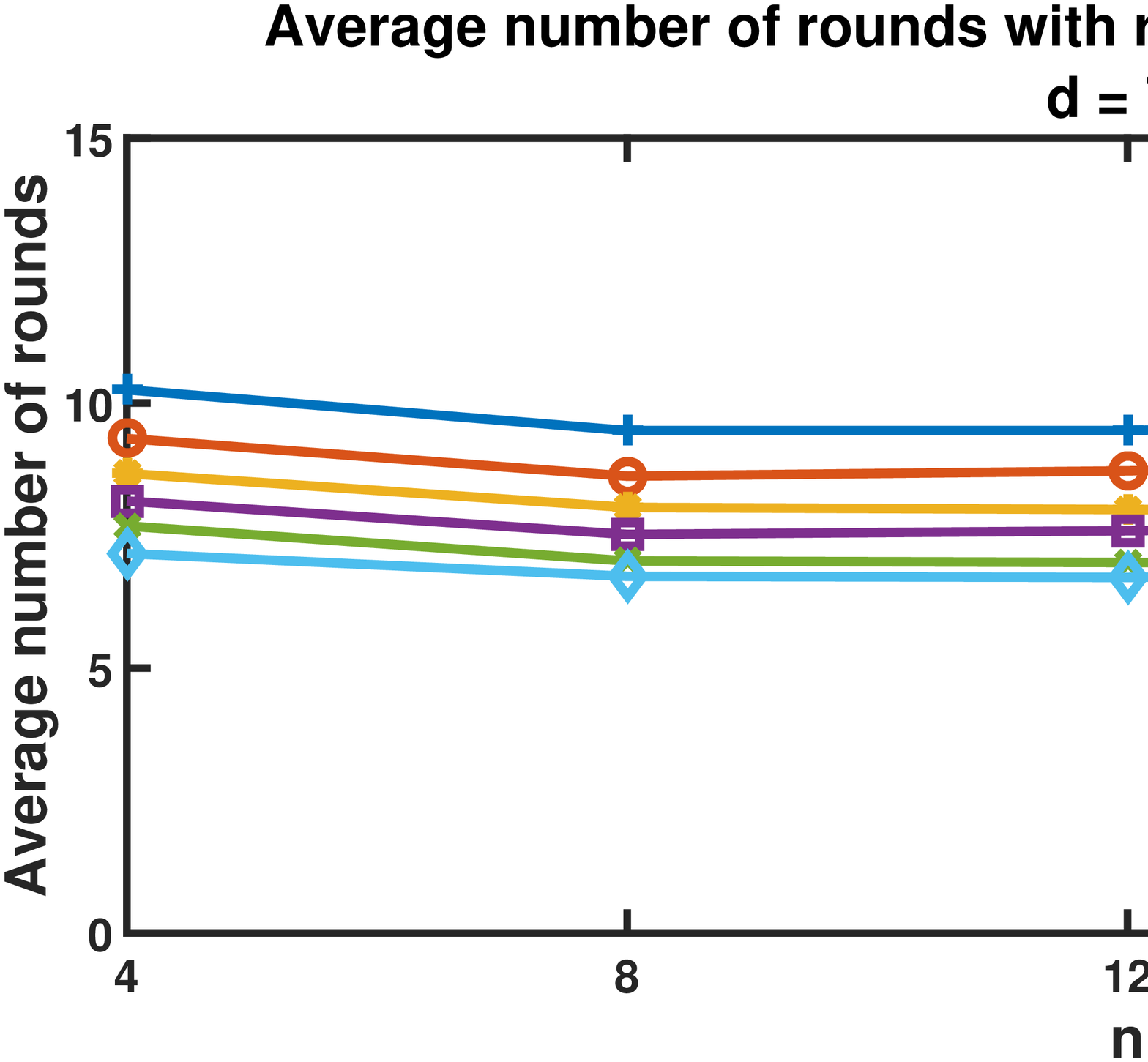}
\caption{Average of total number of rounds.}
\label{fig:simulationsgrp2}
\end{figure}

\begin{figure}[!ht]
\centering
\includegraphics[width=0.49\columnwidth]{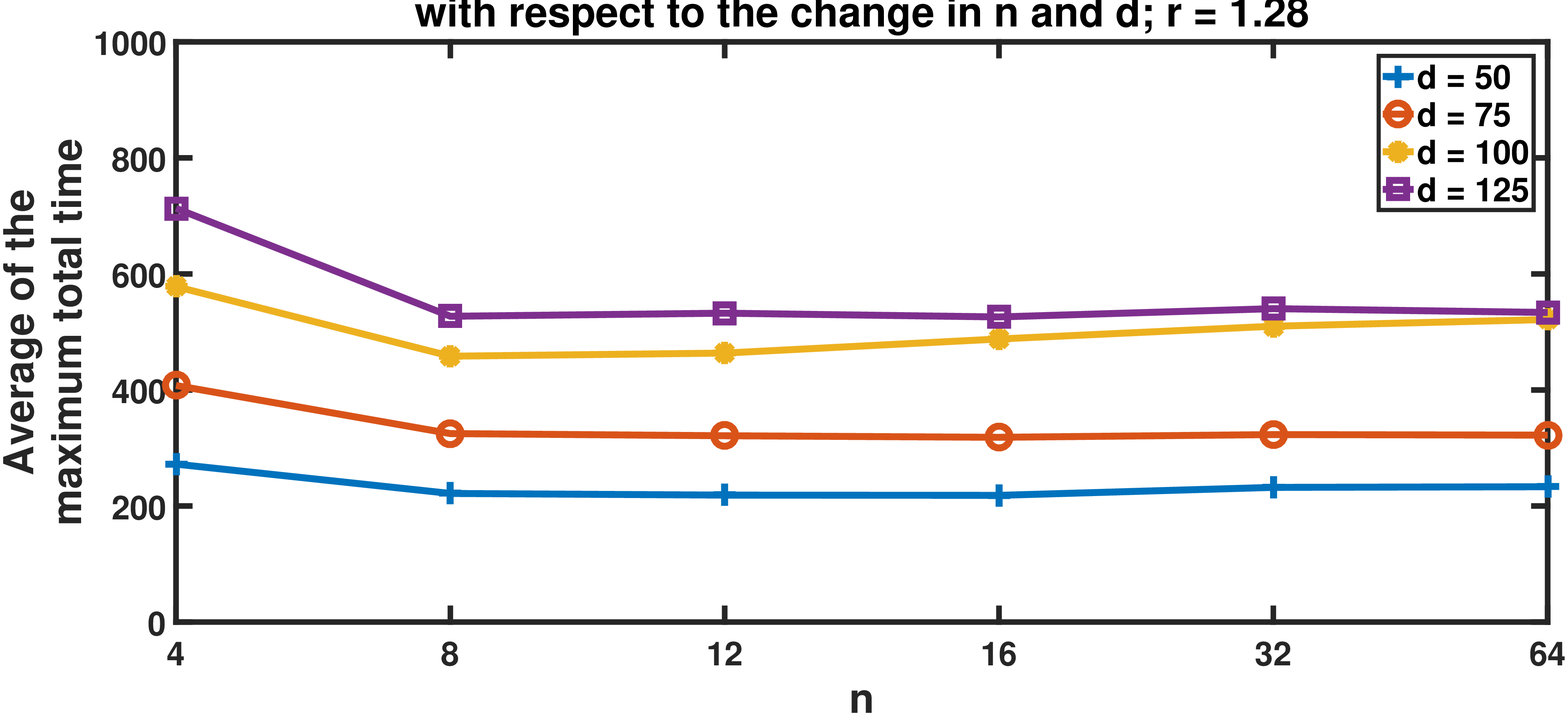}
\includegraphics[width=0.49\columnwidth]{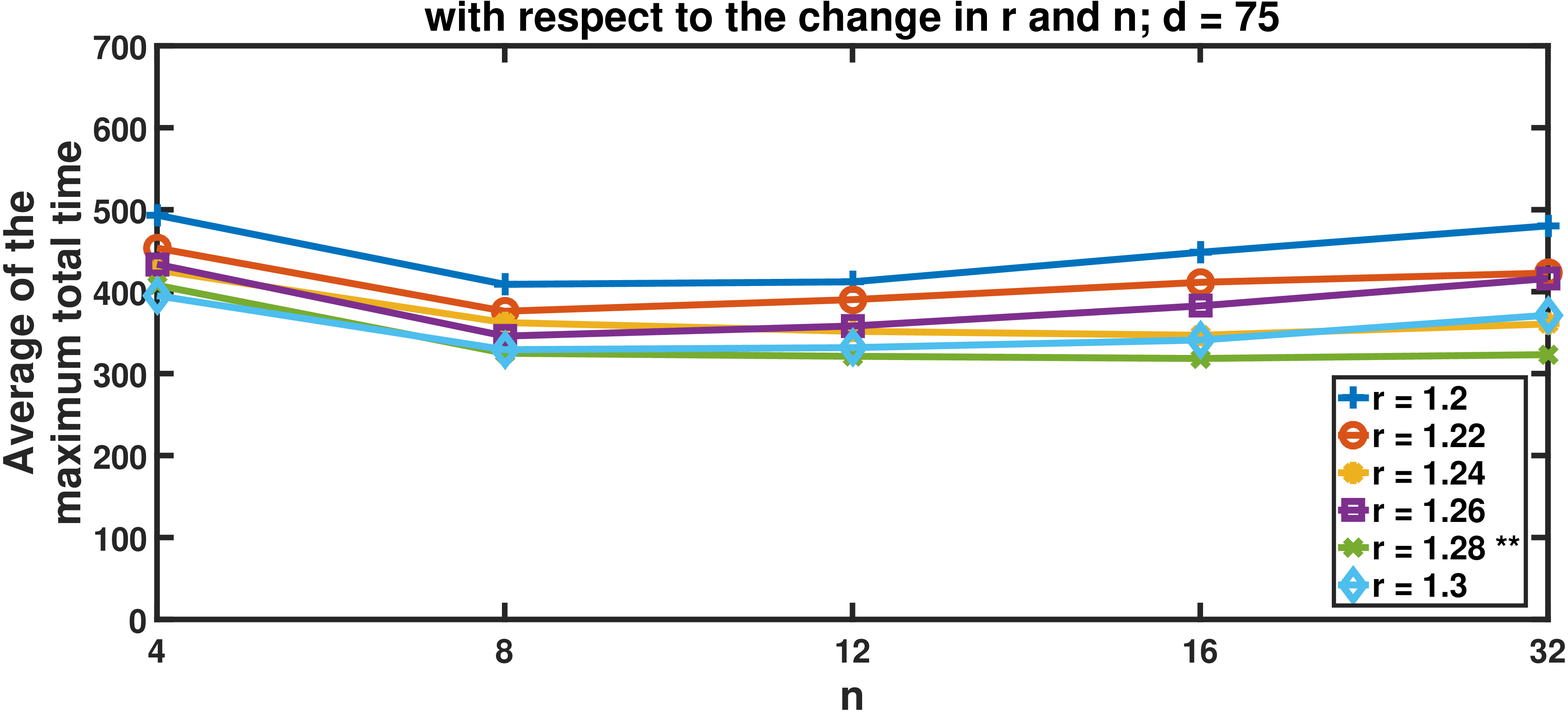}
\caption{Average total time.}
\label{fig:simulationsgrp3}
\end{figure}

\begin{figure}[!ht]
\centering
\includegraphics[width=0.49\columnwidth]{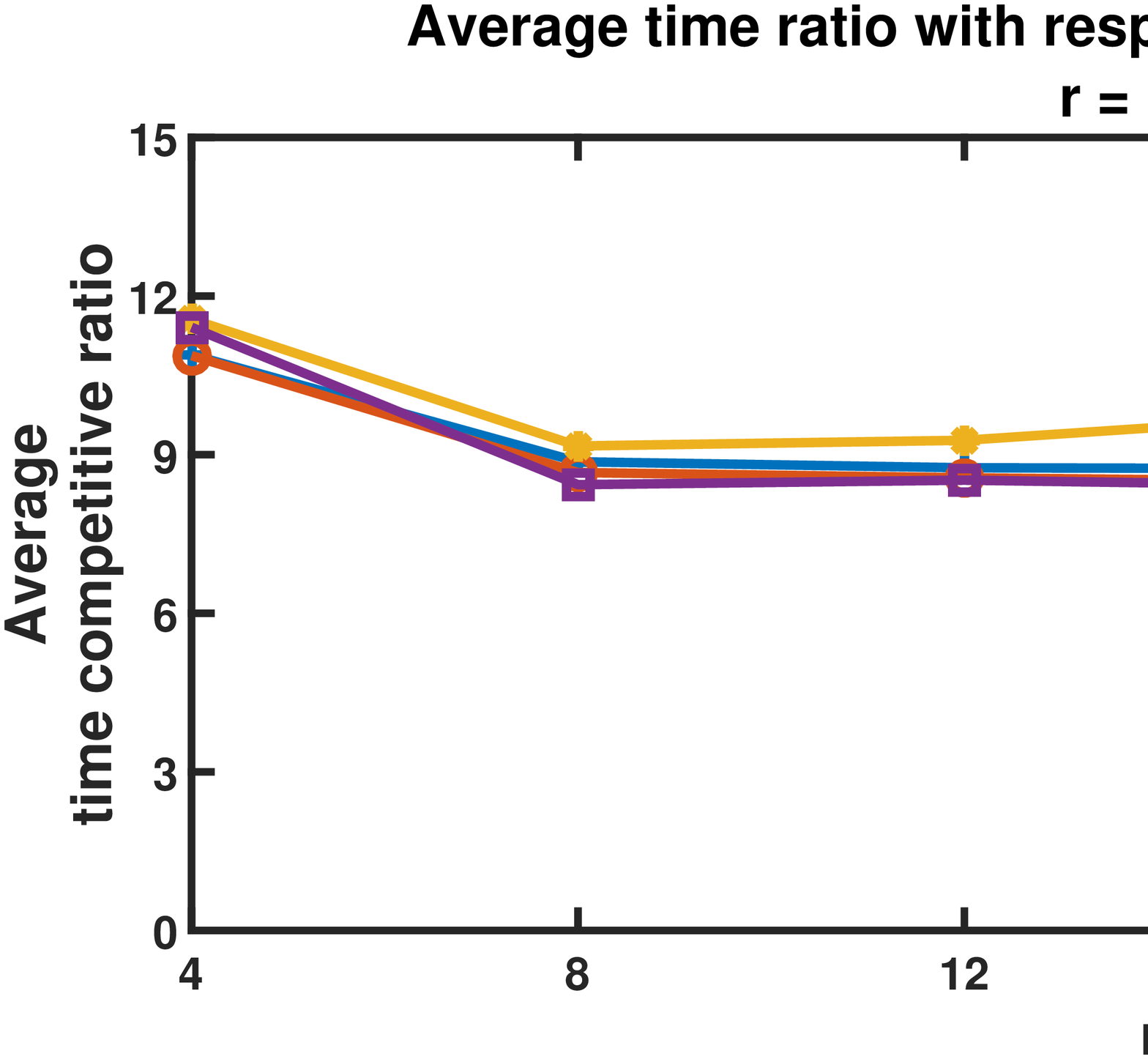}
\includegraphics[width=0.49\columnwidth]{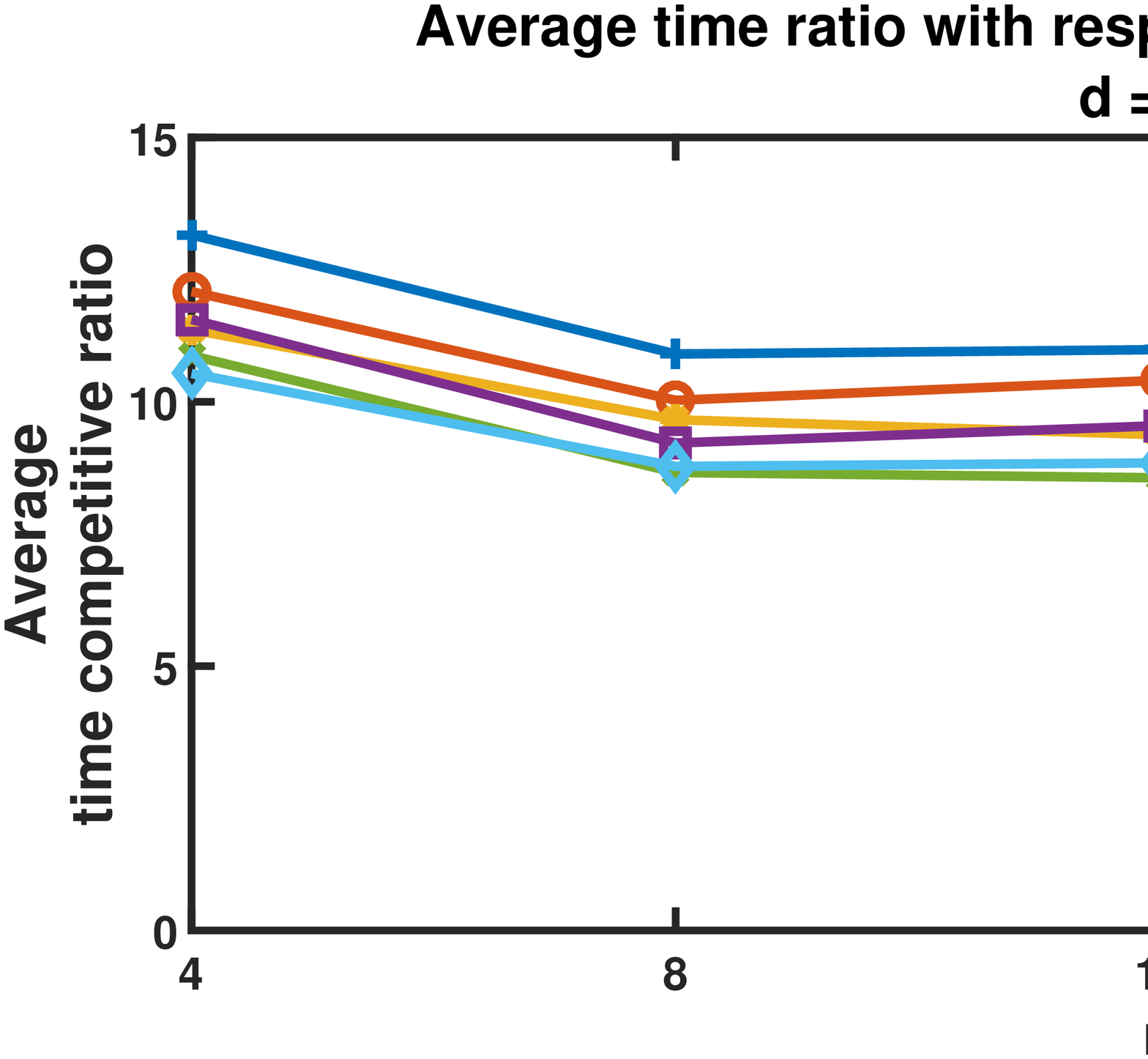}
\caption{Average time competitive ratio.}
\label{fig:simulationsgrp4}
\end{figure}

\begin{figure}[!ht]
\centering
\includegraphics[width=0.49\columnwidth]{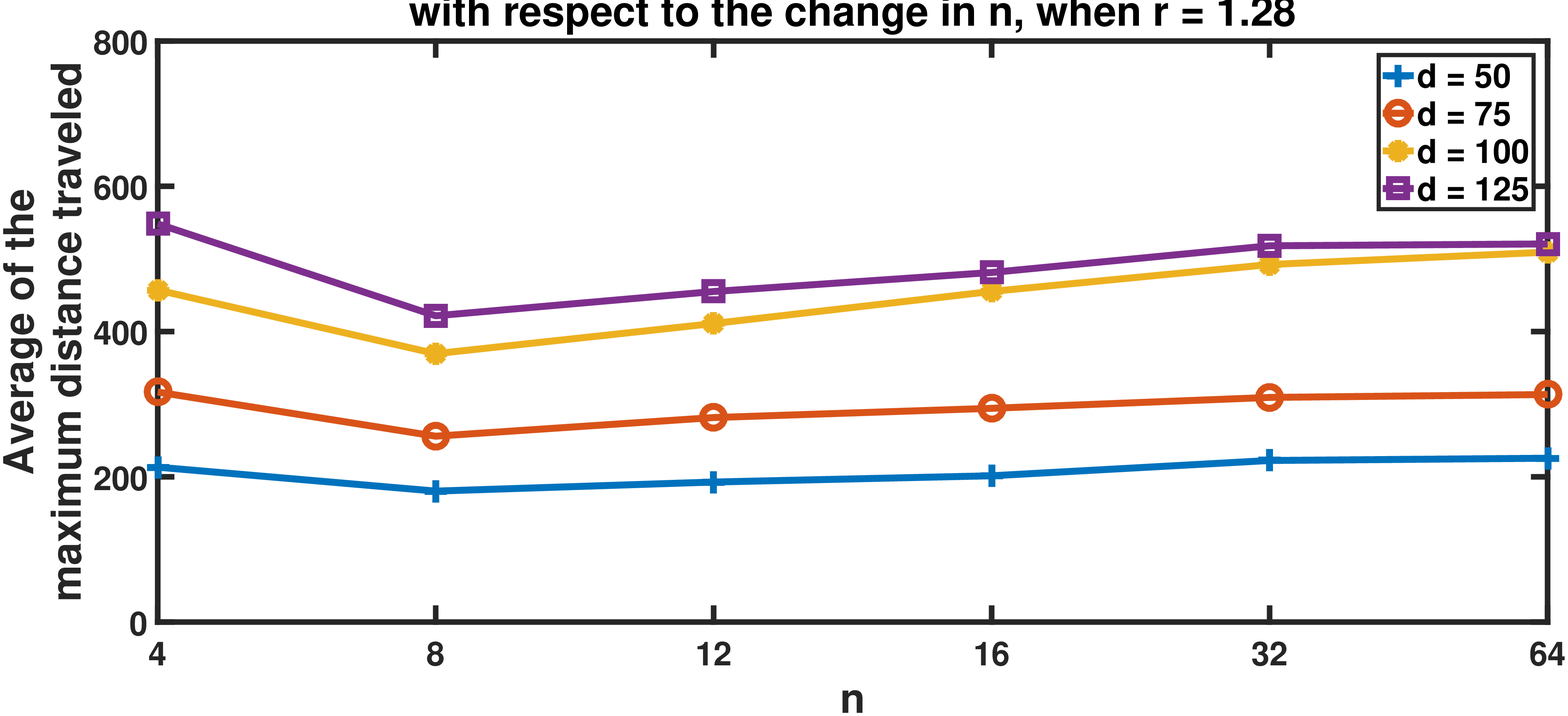}
\includegraphics[width=0.49\columnwidth]{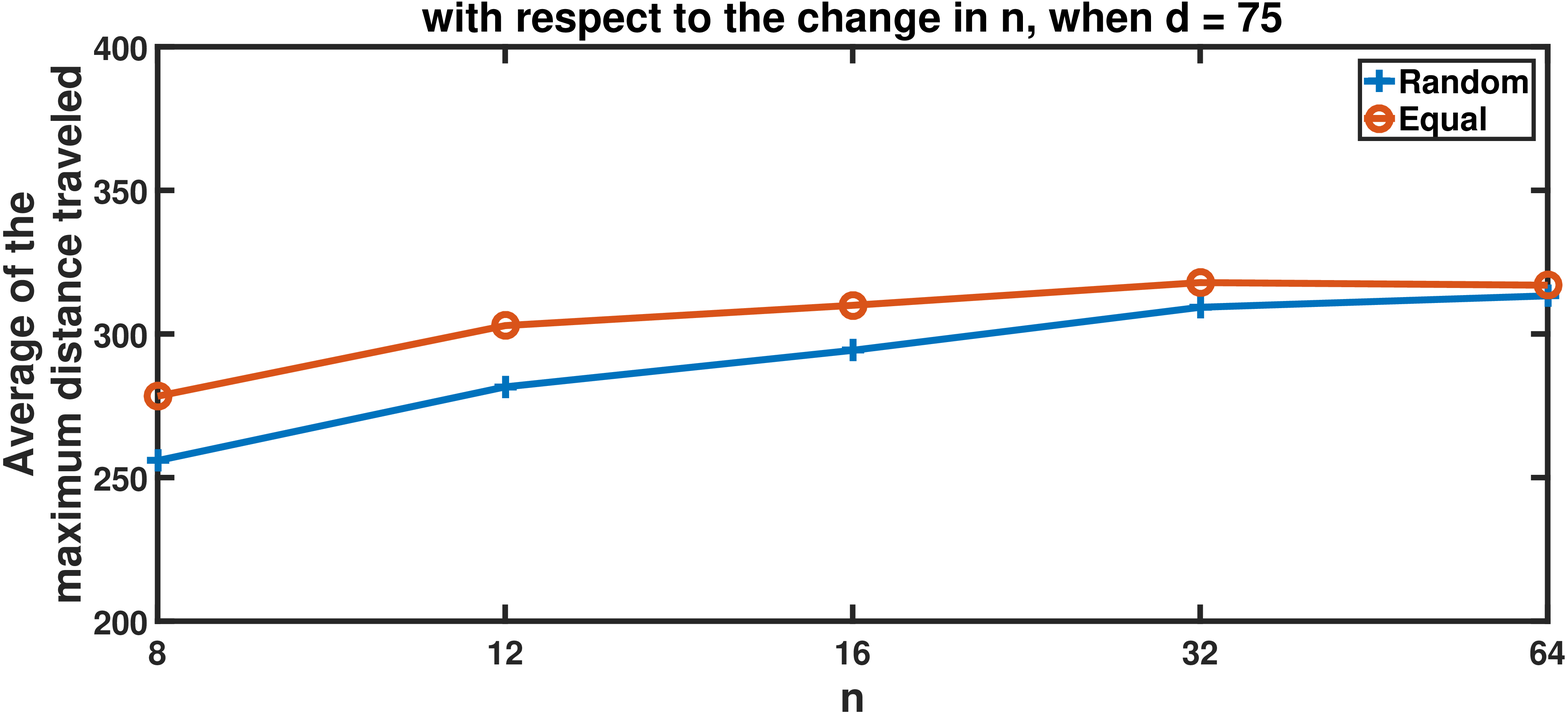}
\caption{Average distance traveled.}
\label{fig:simulationsgrp5}
\end{figure}
\section{Conclusion}
\label{sec:conclusion}

This paper addresses the symmetric rendezvous search problem with multiple robots on the line. In our problem formulation, the initial distance between any pair of robots is unknown to the robots. Moreover, the robots do not know their own positions and the positions of each other. For this problem, we introduced an online algorithm whose competitive ratio is $57.732$ which in the asymptotic case becomes $34.154$.

The algorithm presented here can be extended to the asynchronous case. In the asynchronous case, the robots do not need to start at the same time and use waiting time within the rounds. Our immediate future work is to present this extension and provide its theoretical performance guarantees and simulations.
\section*{Acknowledgement}

The authors would like to thank Christoph Lenzen of the Max Planck Institute for Informatics for his valuable suggestions.

\bibliographystyle{splncs04}
\bibliography{refs}

\end{document}